\documentclass[conference]{IEEEtran}

\usepackage{cite}
\usepackage{amsmath,amssymb,amsfonts}
\usepackage{algorithmic}
\usepackage{graphicx}
\usepackage{textcomp}
\usepackage{xcolor}
\usepackage{amsthm}
\usepackage{url}
\newtheorem{theorem}{Theorem}

\def\BibTeX{{\rm B\kern-.05em{\sc i\kern-.025em b}\kern-.08em
    T\kern-.1667em\lower.7ex\hbox{E}\kern-.125emX}}
\begin{document}

\title{Accelerating Skill Assessment in Chess:\\ A Drift-Diffusion-Enhanced Elo Rating System
\thanks{\copyright~2026 IEEE. Personal use of this material is permitted.
Permission from IEEE must be obtained for all other uses, in any current
or future media, including reprinting/republishing this material for
advertising or promotional purposes, creating new collective works, for
resale or redistribution to servers or lists, or reuse of any copyrighted
component of this work in other works.}
}

\author{\IEEEauthorblockN{Tianyuan Zhou\textsuperscript{*}}
\IEEEauthorblockA{
\textit{School of Intelligent Software}\\
\textit{and Engineering} \\
\textit{Nanjing University}\\
Suzhou, China\\
tyzhou@smail.nju.edu.cn}
\and
\IEEEauthorblockN{Zhizheng Fu\textsuperscript{*}}
\IEEEauthorblockA{
\textit{School of Intelligence Science}\\
\textit{and Technology} \\
\textit{Nanjing University}\\
Suzhou, China\\
zhizheng.fu@smail.nju.edu.cn}
\and
\IEEEauthorblockN{Tianming Yang}
\IEEEauthorblockA{
\textit{Center for Excellence in Brain Science}\\
\textit{and Intelligence Technology} \\ 
\textit{Institute of Neuroscience} \\
\textit{Chinese Academy of Sciences}\\
Shanghai, China \\
tyang@ion.ac.cn}

\thanks{\textsuperscript{*}These authors contributed equally to this work.}
}

\IEEEoverridecommandlockouts

\IEEEpubid{\makebox[\columnwidth]{979-8-3315-9476-3/26/\$31.00~\copyright2026 IEEE\hfill}
\hspace{\columnsep}\makebox[\columnwidth]{}}

\maketitle
\IEEEpubidadjcol

\begin{abstract}
Rating systems such as Elo serve as the gold standard for matchmaking in competitive chess. However, they inherently suffer from response lag due to their exclusive reliance on match outcomes, neglecting the granular quality of gameplay. Nevertheless, incorporating move-by-move information into rating adjustments presents a significant challenge given the substantial noise and the vastness of the game-state space. To address this, we propose the Drift-Diffusion-Enhanced Elo Rating System (DD-Elo), a novel skill assessment framework inspired by the drift diffusion model (DDM) from cognitive neuroscience. By modeling skill expression as a decision-making process, our model integrates move-level data to capture rapid skill fluctuations. We provide a rigorous mathematical derivation proving that DD-Elo maintains a bounded deviation from the traditional Elo system, ensuring theoretical alignment. Extensive experiments demonstrate that DD-Elo adapts to skill changes faster than Elo. Our findings suggest that DD-Elo offers an explainable, highly responsive, and backward-compatible solution for chess rating ecosystems. The implementation code is publicly available at \url{https://github.com/Aquila-zhou1/DD-Elo}.
\end{abstract}


\section{Introduction}

Accurate skill assessment is a cornerstone of competitive chess, underpinning matchmaking quality, ranking fairness, and long-term player engagement. For decades, the Elo rating system has served as the de facto standard for estimating player strength in chess and other competitive games\cite{glickman1995comprehensive,1971149384795592101}. Despite its widespread adoption and conceptual elegance, Elo mostly relies on match outcomes, updating player ratings based solely on win–loss–draw results. While this outcome-centric design ensures robustness and interpretability, it introduces a well-known limitation: slow responsiveness to changes in player skill, particularly in non-stationary settings \cite{glickman1993paired,maitra2025empirical}. This limitation becomes especially salient in practical scenarios such as new player calibration, returning veterans regaining form, or players undergoing intensive training. 

Existing rating system extensions partly address this problem. For example, Glicko\cite{glickman1995glicko} introduces uncertainty estimates to accelerate convergence, while TrueSkill\cite{herbrich2006trueskill,minka2018trueskill} employs Bayesian inference to handle sparse and evolving data. Despite these methodological advances, such systems remain fundamentally outcome-centric: rating updates are triggered only at the granularity of completed matches. Consequently, when the underlying skill of a player changes, the response speed of these models is inherently limited by the frequency of observed games, leading to an unavoidable match-level adaptation lag.

\begin{figure}
    \centering
    \includegraphics[width=\linewidth]{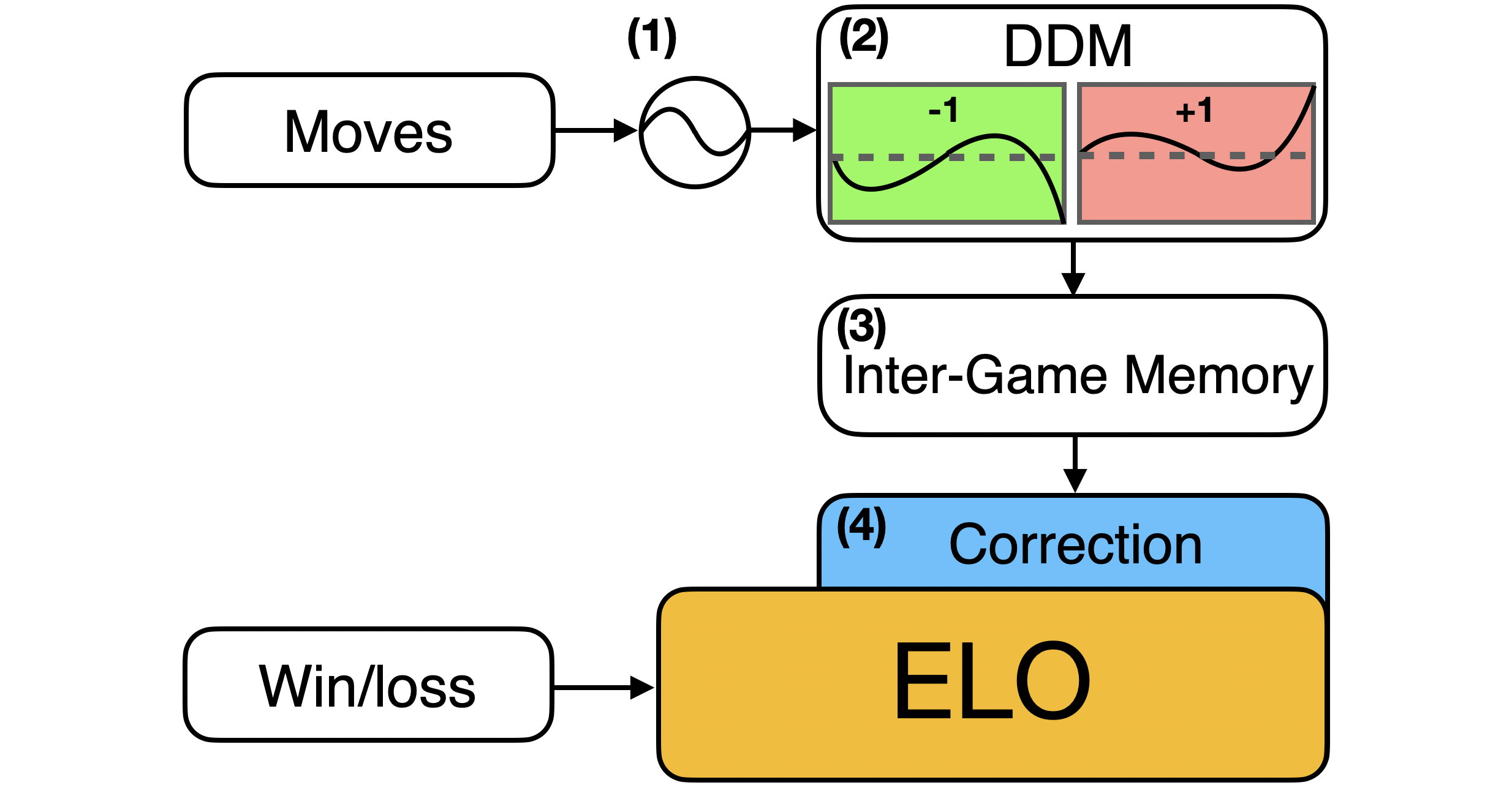}
    \caption{(1) Each move is converted into drift rates and fed into the drift diffusion model. (2) The DDM accumulates evidence for each player to determine whether their moves overperform or underperform relative to their current Elo rating. (3) The decision is then propagated through the inter-game memory mechanism to form a correction. (4) The final result is a combination of the correction term and the traditional win/loss Elo adjustment.}
    \label{fig:struct}
\end{figure}

However, chess is a sequential decision-making process composed of a large number of moves, each reflecting a player’s judgment under uncertainty. Prior research has demonstrated that move-level performance metrics, particularly engine-based evaluations, are strongly correlated with player skill\cite{leite2023expected,guid2006computer}. Analyses of world-class games show that stronger players consistently produce moves closer to engine-optimal choices, and centipawn loss has been established as a quantitative indicator of decision quality across skill levels\cite{leite2023expected}. Engine-based evaluation frameworks, exemplified by Deep Blue\cite{campbell2002deep} and its successors, further reinforce the validity of such fine-grained performance measures.

Despite their undeniable precision in quantifying decision quality, integrating these move-level indicators into a coherent rating system is far from straightforward. Individual moves are highly state-dependent, embedded in an enormous combinatorial game tree, and contaminated by significant stochasticity. Theoretical and empirical studies in sequential decision-making highlight the difficulty of extracting stable signals from such high-dimensional processes\cite{omori2024chess}. As a result, directly aggregating move-level metrics risks amplifying noise and destabilizing long-term ratings.

In this work, we bridge this gap by introducing the Drift-Diffusion-Enhanced Elo Rating System (DD-Elo), a rating framework inspired by the drift diffusion model (DDM) from cognitive neuroscience \cite{bogacz2006physics, ratcliff2008diffusion}. An overview of the framework is shown in Fig.~\ref{fig:struct}. DD-Elo reinterprets a chess game as a sequence of micro-decisions in which each move provides incremental evidence about a player’s latent skill. Move-level performance signals induce a within-game diffusion process whose accumulated state modulates the magnitude and direction of rating updates. This design allows ratings to adapt rapidly when consistent high- or low-quality play is observed, while ensuring long-term stability as shown in Fig.~\ref{fig:intro}.

Beyond empirical performance, theoretical alignment with Elo is also a central design objective. Elo and its variants can be viewed as special cases of paired comparison models rooted in the Bradley–Terry framework\cite{davidson1970extending}. Building on dynamic extensions of these models that allow time‑varying skills and uncertainty estimates\cite{glickman1999parameter,coulom2008whole}, we provide a rigorous theoretical analysis showing that DD-Elo maintains a bounded deviation from classical Elo under mild assumptions. Additionally, recent work analyzing Elo as an online estimator under Bradley–Terry and related models further supports the theoretical soundness of such approaches\cite{olesker2024analysis}. This guarantee ensures that DD‑Elo remains compatible with existing chess rating ecosystems and preserves the interpretability of Elo scores.

In summary, DD-Elo is a chess skill rating system that integrates drift diffusion modeling and move-level decision evidence.  We establish a theoretical bounded-deviation guarantee that formally aligns DD-Elo with Elo. Through large-scale experiments on online chess data, we demonstrate that DD-Elo adapts faster to skill changes than traditional outcome-based rating systems while maintaining stable long-term rankings.

\begin{figure}
    \centering
    \includegraphics[width=\linewidth]{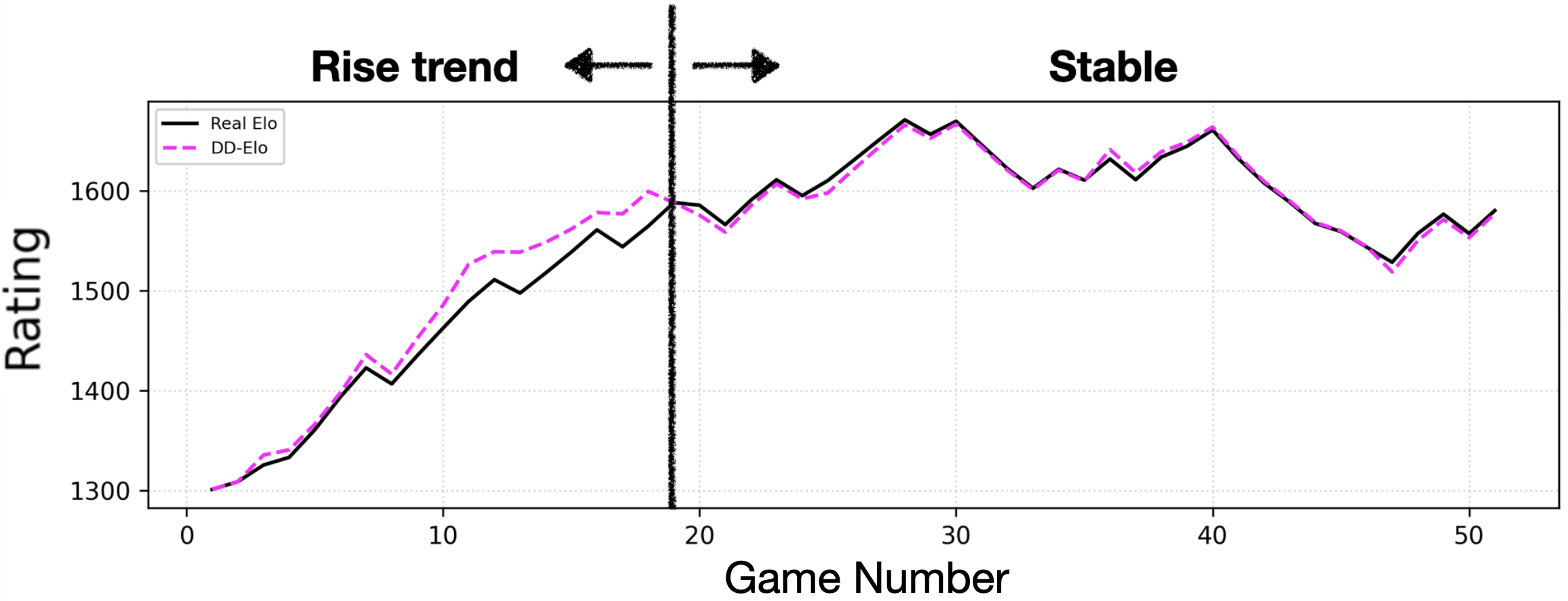}
    \caption{Elo and DD-Elo rating trajectories for a representative player. During a period of rapid skill improvement, DD-Elo (the pink dotted line) responds earlier and yields higher ratings than Elo (the black polyline). When the player’s skill stabilizes, the two trajectories converge, indicating improved responsiveness without sacrificing long-term consistency.}
    \label{fig:intro}
\end{figure}

\section{Related work}

\subsection{Elo-Based Rating Systems in Chess}

The Elo rating system remains the most influential and widely adopted framework for estimating player skill in chess\cite{glickman1995comprehensive,1971149384795592101}. Its success stems from conceptual simplicity, interpretability, and long-term stability. Numerous extensions have been proposed to address practical limitations of the original formulation. Chessmetrics incorporates temporal weighting to better reflect historical dominance and peak performance, while the German Evaluation Number (DWZ)\cite{DSB_DWZ} introduces age-dependent adjustments and performance-based modifiers to improve rating fairness for junior and senior players. Glickman \cite{glickman1995glicko,glickman2012example} further extends Elo by explicitly modeling rating uncertainty and volatility, enabling faster convergence when limited data are available.

\subsection{Drift Diffusion Models in Decision-Making}

The drift diffusion model (DDM) originated in cognitive neuroscience as a normative model of human decision-making under uncertainty\cite{bogacz2006physics,ratcliff2008diffusion,shadlen2006speed}. In its classical form, DDM describes how noisy evidence accumulates over time toward decision boundaries, capturing accuracy–speed trade-offs and trial-to-trial variability. Under standard assumptions, DDM represents a mathematically optimal decision policy. Bogacz et al. \cite{bogacz2006physics} proved that when the drift equals the expected log-likelihood ratio (LLR), DDM is equivalent to the Sequential Probability Ratio Test (SPRT). By Wald’s theorem, SPRT minimizes the expected sample size for given error probabilities, theoretically justifying our use of a discrete DDM to achieve optimal convergence speed. Due to its interpretability and strong empirical grounding, DDM has become a standard tool for modeling perceptual choice, economic decisions, learning, memory, and many other cognitive processes in both humans and animals, and its neural correlates have been revealed in the brain\cite{SHADLEN2013791,Kira2015}. 

\subsection{Move-Level Evaluation in Chess}

Move-level information has long been recognized as a rich source for estimating player skill. Early statistical approaches modeled move quality probabilistically in order to infer intrinsic strength from individual decisions rather than outcomes alone. For example, intrinsic chess rating models \cite{regan2011intrinsic} estimate player skill by measuring deviations from optimal play, while supervised learning methods, such as random forest models, have been used to predict Elo ratings from features extracted from move sequences within single games.

The rapid advancement of chess engines has further improved the reliability of move-level evaluation. Modern engines, from Deep Blue \cite{campbell2002deep} and Rybka \cite{secelle2007golden} to contemporary systems such as Stockfish \cite{stockfish}, provide increasingly accurate assessments of move quality, commonly expressed as centipawn loss. These signals are widely used for post-game analysis and training. However, existing applications primarily support retrospective analysis or static skill prediction and are not integrated into dynamic rating update mechanisms. In contrast, our approach incorporates centipawn-based move-level signals into a drift diffusion process, allowing such information to directly influence rating updates

\section{Preliminaries}

\subsection{Elo Rating System}

The Elo rating system models player skill as a scalar quantity updated after each match based on the difference between observed and expected outcomes. Let $R_t$ denote a player’s rating before a game, and let $S_t \in \{0, 0.5, 1\}$ represent the game outcome (loss, draw, win). The expected score $E_t$ against an opponent with rating $R_t^{\text{opp}}$ is given by
\begin{equation}
    E_t = \frac{1}{1 + 10^{(R_t^{\text{opp}} - R_t)/400}}
\end{equation}
The standard Elo update rule is
\begin{equation}
    R_{t+1} = R_t + K \cdot (S_t - E_t) \label{ELO_eq}
\end{equation}
where $K > 0$ is a fixed update coefficient. 

\subsection{Move-Level Performance Measure: Centipawn Loss}

To quantify move-level performance in chess, we adopt centipawn loss (CPL), a widely used metric derived from chess engine evaluations. Let $\text{Eval}(s)$ denote the engine evaluation of a board position $s$ measured in centipawns. For a given move $m$, the centipawn loss is defined as
\begin{equation}
\text{CPL}_m = \text{Eval}(s^*) - \text{Eval}(s_m)
\end{equation}
where $s^*$ is the position resulting from the engine-recommended best move, and $s_m$ is the position after executing move $m$.

\subsection{Drift Diffusion Model}

The drift diffusion model (DDM) models decision formation as the accumulation of noisy evidence over time, driven by an underlying drift signal and random fluctuations.

In discrete time, the evolution of the latent evidence state $X_m$ is described by
\begin{equation}
X_m = X_{m-1} + v_m + \epsilon_m
\end{equation}
where $X_0 = 0$, $v_m$ denotes the instantaneous evidence at time $m$, and $\epsilon_m \sim \mathcal{N}(0,\sigma_m^2)$ is the noise term capturing stochastic variability in evidence accumulation. 

The latent evidence $X_m$ evolves until it reaches one of the absorbing decision boundaries $\pm \beta$, at which point the process terminates at $\tau = \inf \left\{ m \ge 1 : |X_m| \ge \beta \right\}$. 

\subsection{Theoretical Optimality of DDM}

DDM provides a principled mechanism for aggregating noisy, sequential evidence into a single decision variable. Specifically, when the discrete drift $v_m$ equals the log-likelihood ratio of an observation $z_m$ under two hypotheses ($H_1, H_2$) given independent and identically distributed observations, the boundary accumulation mathematically reduces to the Sequential Probability Ratio Test (SPRT). Unlike fixed-sample tests, SPRT evaluates the joint likelihood ratio after each observation against preset thresholds to accept $H_1$, accept $H_2$, or continue sampling \cite{bogacz2006physics}:
\begin{equation}
\begin{aligned}
    X_T & = \sum_{m=1}^T v_m = \sum_{m=1}^T \log \frac{P(z_m \mid H_1)}{P(z_m \mid H_2)} \equiv \text{SPRT}
\end{aligned}
\end{equation}

SPRT has been proved to possesses strict optimality, minimizing the expected sample size required for a decision under any given error probabilities\cite{bogacz2006physics}. This exact mathematical equivalence guarantees that DDM minimizes the expected number of decision steps for any given error boundaries, rendering it a theoretically optimal decision policy.

\section{Algorithmic Design}

Our goal is to incorporate move-level performance signals into the rating update process while preserving stability and compatibility with classical Elo. We achieve this through four components: move-level drift construction, in-game diffusion, inter-game memory decay, and the final rating update rule.

\subsection{Move-Level Drift Construction}

We begin by defining how evidence for a player’s latent skill is computed from each individual move. Consider a game $t$ consisting of a sequence of moves indexed by $m$. For each move $m$, we define a move-level drift value $v_m$, which serves as the instantaneous evidence increment in the diffusion process.

We define
\begin{equation}
    v_m = f\left( \mathrm{CPL}_m - \mathcal{G}(R_t)\right) \cdot s_m \cdot \pi_m
\end{equation}
where $f(\cdot)$ is a monotonically decreasing function defined as
\begin{equation}
    f(t) = 10 - 10 \log\big(1 + 0.1t\big),
    \qquad t \ge 0
\end{equation}
which smoothly downweights large move errors while preserving sensitivity to small deviations.
The term $\mathcal{G}(\cdot)$ represents the expected centipawn loss of a player with rating $R_t$; $s_m \in \{+1, -1\}$ indicates whether the player is the active decision-maker at move $m$; $\pi_m = \min\big(E_t, E_t^{\mathrm{opp}}\big)$ represents the confidence of the move, defined as the minimum of the expected scores of the player and the opponent.

\subsection{In-Game Drift Diffusion Process with Decision Boundaries}

We model a player’s performance within a single game $t$ as a discrete-time drift diffusion process with absorbing boundaries. Let $X_{t_m}$ denote the latent evidence state after move $m$. The process evolves according to
\begin{equation}
X_{t_m} = X_{t_{m-1}} + v_{t_m}
\end{equation}
where $v_{t_m}$ represents the move-level drift, which contains stochastic variability in the player’s decision-making.

A decision occurs whenever the latent evidence reaches the predefined boundaries $\pm \beta$. Upon reaching a boundary, a reward $A$ is recorded.

Let $D_t$ denote the cumulative evidence extracted from all in-game decisions up to the end of the game:
\begin{equation}\label{eq:D_t}
D_t = \sum_{k=1}^{N_t} \mathrm{sign}(X_{\tau_k}) \cdot A 
\end{equation}

Here, $N_t$ denotes the number of boundary crossings up to time $t$, $\tau_k$ is the time of the $k$-th boundary crossing, and $X_{\tau_k}$ is the latent evidence at that moment. The function $\mathrm{sign}(\cdot)$ indicates the direction of the boundary crossing. 

\subsection{Inter-Game Memory Decay}

To prevent unbounded accumulation of evidence across games and to preserve long-term stability, we introduce an inter-game memory mechanism. Let $\Delta_t$ denote the long-term diffusion memory after game $t$. The memory is updated as
\begin{align}
\Delta_{t+1} &= \lambda \Delta_t + D_t\label{eq:memory_update} \\
X_{[t+1]_{0}} &= \lambda X_{t_m}\label{eq:residual_decay}
\end{align}
where $\lambda \in (0,1)$ is a decay factor controlling how quickly past evidence fades.

\subsection{DD-Elo Rating Update Rule}

Finally, DD-Elo modifies the standard Elo dynamics by introducing an additive  correction term. Let $R_{t+1}$ denote the post-game rating obtained
from the classical Elo update \eqref{ELO_eq}; the corrected rating is given by
\begin{equation}
    \tilde{R}_{t+1} = R_{t+1} + \Delta_{t+1}
\end{equation}

Here, the first term corresponds to the standard Elo adjustment, while the additional diffusion term acts as a performance-sensitive correction derived from move-level evidence. 

\section{Theoretical Analysis}\label{TA}

In this section, we establish the mathematical foundations of the DD-Elo model by analyzing its core theoretical properties. First, we prove that DD-Elo maintains a strictly bounded numerical deviation from the classical Elo system (Section \ref{th:A}). Building upon this boundedness property, we then demonstrate the theoretical alignment between the two systems, ensuring the global stability of the ranking structure (Section \ref{th:B}). Finally, we formally prove that the DDM architecture in DD-Elo represents the mathematically optimal decision policy for sequential rating updates (Section \ref{th:C}).

\subsection{Bounded Numerical Deviation}
\label{th:A}

Let $\Delta_t = \tilde R_t - R_t$ denote the difference between DD-Elo and classical Elo. According to (\ref{eq:D_t}), since each game contains a finite number of decisions, we have $|D_t| \le N_t |A| \le c$, where $c := N_{\max}|A|$.

\begin{theorem}[Bounded Deviation from Elo]
\label{thm:bounded_deviation}
If $|D_t| \le c$, then according to (\ref{eq:memory_update}), the deviation remains uniformly bounded:
\[
|\Delta_t| \le C_0,
\qquad
C_0 := \frac{c}{1-\lambda}
\]
Consequently, $|\tilde R_t - R_t| \le C_0$ for all t.
\end{theorem}

\begin{proof}
From the update rule (\ref{eq:memory_update}) we have $\Delta_{t+1} = \lambda \Delta_t + D_t$. Using $|D_t|\le c$ gives $|\Delta_{t+1}| \le \lambda |\Delta_t| + c$. Iterating this recursion yields
\[
|\Delta_t|
\le
\lambda^t |\Delta_0|
+
c\sum_{k=0}^{t-1}\lambda^k
=
\lambda^t |\Delta_0|
+
\frac{c(1-\lambda^t)}{1-\lambda}
\]
Taking $t\to\infty$ implies $|\Delta_t| \le \frac{c}{1-\lambda}$.
\end{proof}

\subsection{Spearman Rank Robustness under Bounded Perturbations}
\label{th:B}

We analyze how a bounded perturbation of player ratings affects the
global ranking structure measured by Spearman's rank correlation.

\begin{theorem}[Spearman Robustness]
Let $R_1,\dots,R_n$ denote the Elo ratings of $n$ players and suppose the
DD-Elo ratings satisfy
\[
\tilde R_i = R_i + \delta_i,
\qquad
|\delta_i| \le C_0
\]
for all $i=1,\dots,n$.
Let $\rho_S$ denote the Spearman rank correlation coefficient between
$\{R_i\}$ and $\{\tilde R_i\}$.
Assume that the rating distribution admits a density function
$f_R$ satisfying $f_R(x)\le M$.
Then
\[
1-\rho_S
=
O\!\left(\frac{C_0^3}{\sigma_R^3}\right)
\]
where $\sigma_R$ denotes the standard deviation of the rating distribution.
\end{theorem}

\begin{proof}
From Theorem~\ref{thm:bounded_deviation}, we have
\[
|\tilde R_i-R_i|\le C_0
\]
Consequently,
\[
\tilde R_i \ge R_i-C_0,
\qquad
\tilde R_j \le R_j+C_0
\]
Therefore, if $|R_i-R_j|>2C_0$, the relative ordering of players $i$ and $j$ cannot change after the perturbation. Hence, rank inversions may occur only for pairs satisfying $|R_i-R_j|\le 2C_0$.

Since the rating distribution has bounded density $f_R(x)\le M$, for any
fixed player $i$,
\[
P\bigl(|R_i-R_j|\le 2C_0\bigr)
\le 4MC_0
\]
Thus only an $O(MC_0)$ fraction of player pairs are capable of producing
rank inversions.

The Spearman rank correlation coefficient is given by
\[
\rho_S
=
1-
\frac{6}{n(n^2-1)}
\sum_{i=1}^{n} d_i^2
\]
where $d_i$ denotes the difference between the ranks of player $i$
under the two rating systems.

For a given player, the expected number of nearby opponents whose order
may change is $O(nMC_0)$. Since these potentially inverted opponents are
contained within an interval of width $O(C_0)$, the corresponding rank
displacement satisfies $d_i = O(nMC_0)$. Combining this with the fraction $O(MC_0)$ of players involved in possible inversions yields
\[
\sum_{i=1}^{n} d_i^2
=
4MC_0\, n \cdot O\!\left(n^2M^2C_0^2\right)
=
O\!\left(M^3C_0^3n^3\right)
\]

Substituting this estimate into the definition of $\rho_S$ gives
\[
1-\rho_S
=
O\!\left(M^3C_0^3\right)
\]

Finally, the density scale is constrained by the spread of the rating
distribution. In particular, $M = O\!\left(\frac{1}{\sigma_R}\right)$.
Substituting this relation yields
\[
1-\rho_S
=
O\!\left(\frac{C_0^3}{\sigma_R^3}\right)
\]
which completes the proof.
\end{proof}

\subsection{Theoretical Optimality of the DD-Elo Model}
\label{th:C}

This section proves that, under specific rating error constraints, the drift-diffusion updating mechanism in DD-Elo minimizes the objective loss function with the fewest expected moves.

\subsubsection{Problem Formulation and Hypotheses}

Since the Elo rating system is symmetric and zero-sum, rating adjustments can be modeled as a two-alternative forced-choice (2AFC) problem between two hypotheses. Under $H_1$ (rating increase), the player is assumed to be underrated and the opponent overrated, leading to a lower expected player error ($\mu_{p1}$) and a higher one for the opponent ($\mu_{o1}$). Under $H_2$ (rating decrease), the player is overrated and the opponent underrated, yielding a higher player error $\mu_{p2}$ ($\mu_{p2}>\mu_{p1}$) and a lower opponent error $\mu_{o2}$ ($\mu_{o2}<\mu_{o1}$).

Let $R_{true}$ be the true rating, $\overline{R}$ be the estimated rating, and $N$ be the decision steps. DD-Elo aims to minimize the expected risk cost functional $\mathcal{J}$:
\begin{equation}\label{eq:loss}
    \min \mathcal{J} = \mathbb{E}[|\overline{R} - R_{true}|] + \delta \cdot \mathbb{E}[N]
\end{equation}
where $\delta > 0$ is the penalty coefficient for convergence delay.

\subsubsection{Joint Log-Likelihood Ratio under Exponential Distribution}

Given the non-negative and heavily right-skewed nature of CPL, we assume the single-move error follows an exponential distribution $CPL \sim \text{Exp}(\lambda)$, where the rate parameter $\lambda = 1/\mu$ and $\mu = \mathbb{E}(R)$.

Under $H_1$ and $H_2$, the optimal joint log-likelihood ratio (LLR) encompassing both players' performances is a linear combination of their respective CPLs.

\begin{proof}
The joint LLR is the log-sum of both players' independent observation probability ratios:
\begin{equation}
    LLR = \log \frac{P(CPL_p \mid H_2)}{P(CPL_p \mid H_1)} + 
    \log \frac{P(CPL_o \mid H_2)}{P(CPL_o \mid H_1)}
\end{equation}
Substituting the exponential probability density function $P(x)=\lambda e^{-\lambda x}$ yields
\begin{equation}
    LLR = C + (\lambda_{p1}-\lambda_{p2}) CPL_p 
    - (\lambda_{o1}-\lambda_{o2}) CPL_o
\end{equation}
Since $\mu_{p2}>\mu_{p1}$ implies $\lambda_{p1}>\lambda_{p2}$ and 
$\mu_{o2}<\mu_{o1}$ implies $\lambda_{o2}>\lambda_{o1}$, the coefficient of $CPL_p$ is positive while that of $CPL_o$ is negative.
\end{proof}

Motivated by this structure, the DD-Elo move-level feature is constructed as $v_m = \omega_p \Phi(CPL_p) - \omega_o \Phi(CPL_o)$ where $\Phi(CPL)=f(\max(CPL-E(R),0))$ extracts right-tail errors and $\omega$ denotes confidence weights. 
This formulation mirrors the positive and negative terms of the optimal LLR while using truncation to focus estimation on statistically informative blunders.

Crucially, as established in the Preliminaries, a DDM achieves theoretical optimality (via equivalence to SPRT) when its drift represents the LLR and its objective is to minimize a cost function balancing accuracy and decision time \cite{bogacz2006physics}. DD-Elo explicitly satisfies all these prerequisites: its accumulated evidence $v_m$ maps directly to the optimal joint LLR, and its threshold mechanism structurally minimizes the expected cost functional $\mathcal{J}$ defined in (\ref{eq:loss}). Therefore, the accumulation mechanism of DD-Elo is not merely a heuristic, but the mathematically optimal decision policy for dynamic rating updates.

\section{Experimental Setup}

\subsection{Dataset}
We conduct experiments on a large-scale real-world chess dataset from Lichess\cite{lichess}, consisting of approximately 10 million rated games played in January 2019, involving 429,000 active players. To ensure sufficient signal quality for skill estimation, we exclude Bullet and Blitz games and analyze only players who completed at least 100 rated games. In addition, we calculate the centipawn loss (CPL) using Stockfish\cite{stockfish} as the move-level performance measurement.

\subsection{Identification of Non-Stationary Phases}

A central challenge in evaluating rating systems is distinguishing genuine skill changes from stochastic fluctuations. Rather than analyzing entire rating histories indiscriminately, we explicitly identify non-stationary phases—periods during which a player’s underlying skill exhibits a sustained upward or downward trend.

\subsubsection{Signal Smoothing}

For each player’s Elo rating trajectory, we first apply a moving average filter with a window size of $W_s$ games, a hyperparameter controlling the smoothness of the trajectory:
\begin{equation}
    \tilde{\text{Elo}}_t = \frac{1}{W_s} \sum_{i=-\lfloor W_s/2 \rfloor}^{\lfloor W_s/2 \rfloor} \text{Elo}_{t+i}
\end{equation}
This step suppresses high-frequency noise while preserving medium-term trends relevant to skill evolution.

\subsubsection{Trend Signal Extraction}

To detect structural changes, we apply a first-order derivative filter with a window size of $W_b$:
\begin{equation}
    \text{Trend}_t = \sum_{i=-\lfloor W_b/2 \rfloor}^{\lfloor W_b/2 \rfloor} i \cdot \tilde{\text{Elo}}_{t+i}
\end{equation}
A non-stationary phase is identified when $|\text{Trend}_t| > \theta$. The rising or declining skill phases are classified by the sign of $\text{Trend}_t$.

\subsection{Evaluation Metrics}

To comprehensively assess fast adaptation performance, we adopt four complementary metrics, each capturing a distinct aspect of responsiveness and reliability.

\subsubsection{Area Improvement Percentage (AIP)}

Area Improvement Percentage (AIP) measures the fraction of rating modification area that is aligned with the underlying skill trend during a monotonic transition phase.

Let $\text{Elo}(t)$ and $\text{DD-Elo}(t)$ denote the rating trajectories over time. The modification introduced by DD-Elo relative to Elo is represented by the absolute area between the two curves over the phase. Given that the underlying skill exhibits a consistent trend, this total area can be decomposed into a trend-aligned (correct) area and a trend-opposed (incorrect) area according to the direction of the skill change.

Formally, let $Area_{\text{correct}}$ denote the absolute area where the DD-Elo modification is directionally consistent with the skill trend, and let $Area_{\text{total}}$ denote the total absolute modification area. AIP is defined as
\begin{equation}
    \text{AIP} = \frac{Area_{\text{correct}}}{Area_{\text{total}}} \times 100\%
\end{equation}

\begin{figure}
    \centering
    \includegraphics[width=\linewidth]{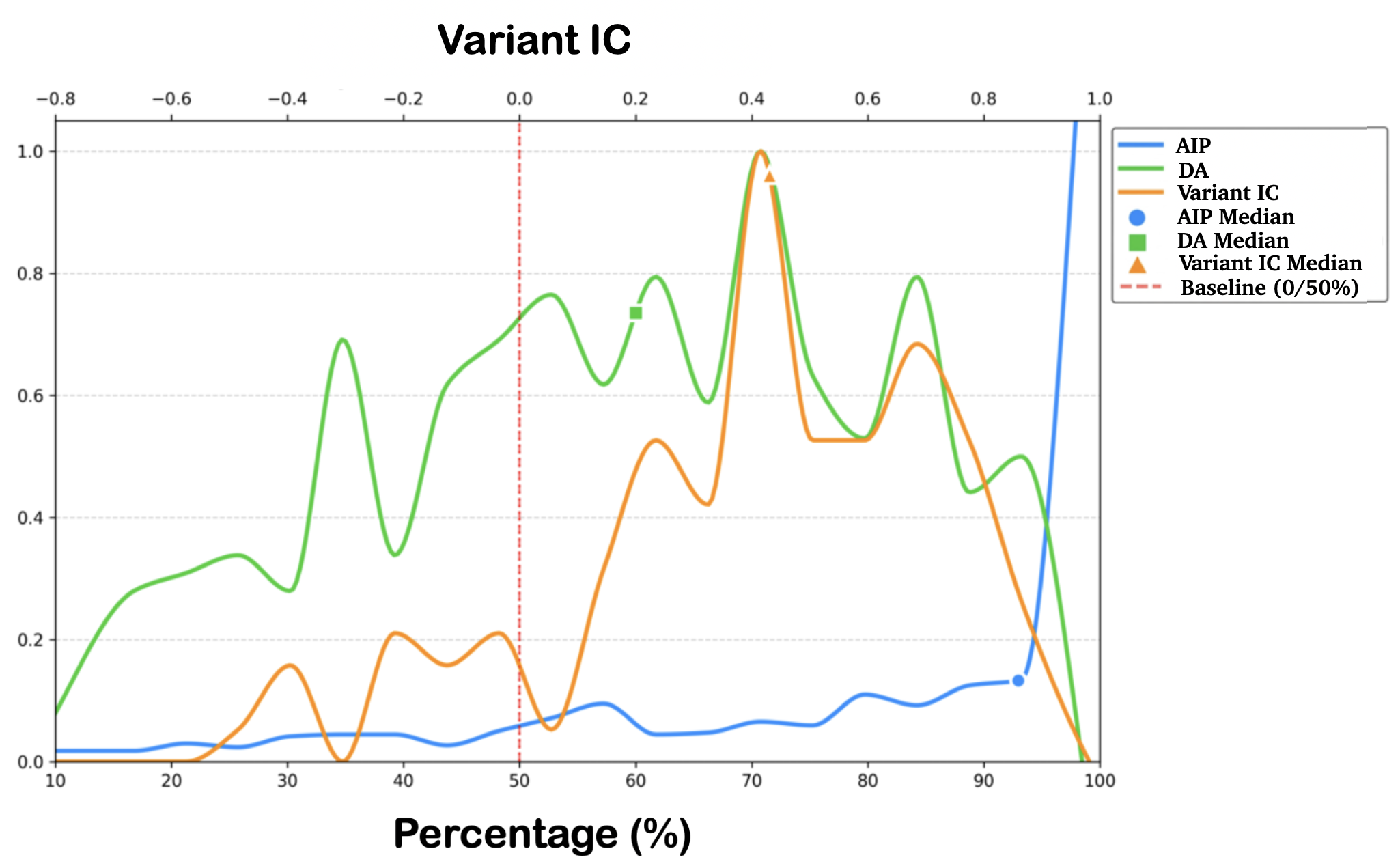}
    \caption{Distribution of Area Improvement Percentage (AIP), Directional Accuracy (DA), and Information Coefficient (IC), with sample counts normalized to have a peak value of 1.}
    \label{fig:sum}
\end{figure}

By construction, an AIP of $50\%$ corresponds to a non-informative baseline in which modifications are equally distributed between trend-aligned and trend-opposed directions. Values above $50\%$ indicate that a larger fraction of the modification area is concentrated in the correct direction, reflecting more effective adaptation, while values below $50\%$ indicate predominantly misaligned adjustments.

\subsubsection{Directional Accuracy (DA)}
In contrast to the area-based nature of AIP, Directional Accuracy (DA) evaluates directional correctness at a qualitative level, independent of modification magnitude.

Directional Accuracy evaluates whether the direction of DD-Elo corrections aligns with the detected trend signal at each time step:
\begin{equation}
    DA = \frac{1}{N} \sum_{t=1}^N \mathbb{I} \Big( \text{Sign}({\Delta_t}) = \text{Sign}({Trend_t}) \Big)
\end{equation}
where $\mathbb{I}(\cdot)$ is the indicator function. DA can be interpreted as the empirical probability that DD-Elo applies corrections in the correct direction. A value of $50\%$ corresponds to random directional agreement, while values above $50\%$ indicate that DD-Elo reliably identifies the direction of change, avoiding systematic counter-corrections during neutral or noisy periods.

\subsubsection{Average Lead Time (ALT)}

Beyond directional quality, Average Lead Time (ALT) explicitly  measures how many games earlier DD-Elo reaches the same rating levels as Elo.

Consider a transition phase with rating range $[E_{\min}, E_{\max}]$, partitioned into $K$ uniformly spaced thresholds
\[
\theta_k = E_{\min} + \frac{k}{K}(E_{\max} - E_{\min}), \quad k = 1, \dots, K
\]
For each threshold $\theta_k$, let $t_{\text{Elo}}^{\theta_k}$ and $t_{\text{DD-Elo}}^{\theta_k}$ denote the first game index at which the respective rating reaches $\theta_k$. ALT is defined as
\begin{equation}
    \text{ALT} = \frac{1}{K} \sum_{k=1}^K \left( t_{\text{Elo}}^{\theta_k} - t_{\text{DD-Elo}}^{\theta_k} \right)
\end{equation}

A positive ALT indicates that DD-Elo reaches the same rating levels earlier than Elo, directly quantifying the temporal lead in terms of the number of games. An ALT of zero corresponds to no temporal advantage.

\subsubsection{Information Coefficient (IC)}

All the metrics discussed previously evaluate the DD-Elo correction term ($\Delta$) with respect to a player's current skill level. A natural question then arises: can we further investigate the predictive capability of $\Delta$? 

To evaluate this relationship, we adopt the Information Coefficient (IC). Originally introduced in portfolio management \cite{ambachtsheer1974profit}, IC measures the correlation between predicted values and realized outcomes. It is defined as the sample Pearson correlation between a factor $X$ and the corresponding outcome $Y$:
\begin{equation}
IC = \frac{\sum_{i=1}^{n} (X_i - \bar{X})(Y_i - \bar{Y})}
{\sqrt{\sum_{i=1}^{n} (X_i - \bar{X})^2 \sum_{i=1}^{n} (Y_i - \bar{Y})^2}}
\end{equation}
where $\bar{X}$ and $\bar{Y}$ are the sample means of $X$ and $Y$.

In quantitative finance, an IC is considered statistically meaningful if its absolute value exceeds a certain threshold. Following the standards suggested in \cite{jung2025alpha}, we adopt $|IC| > 0.02$ as the threshold for significance. An IC value surpassing 0.02 implies that the factor contains significant predictive information regarding the outcome.

\textbf{Standard IC: Direct Prediction}\\
Our first application, denoted as Standard IC, applies the IC formula directly to rating changes. The DD-Elo correction term $\Delta_t$ is treated as the predictive factor ($X$), while the realized outcome ($Y$) is the player's Elo change over the next $k$ games, i.e., $Y = Elo_{t+k} - Elo_t$. A positive Standard IC indicates that $\Delta$ has predictive power for future rating movements.

\textbf{Variant IC: Trend Prediction}\\
Short-term Elo changes are highly stochastic, which can suppress the Standard IC. To better reflect underlying skill dynamics, we introduce a Variant IC that compares $\Delta_t$ with a denoised skill trend (\textit{Signal}) obtained during data processing. In this case the factor remains $X=\Delta_t$, while the outcome is the Signal value representing the direction of the player's skill evolution.

\begin{figure}
    \centering
    \includegraphics[width=0.985\linewidth]{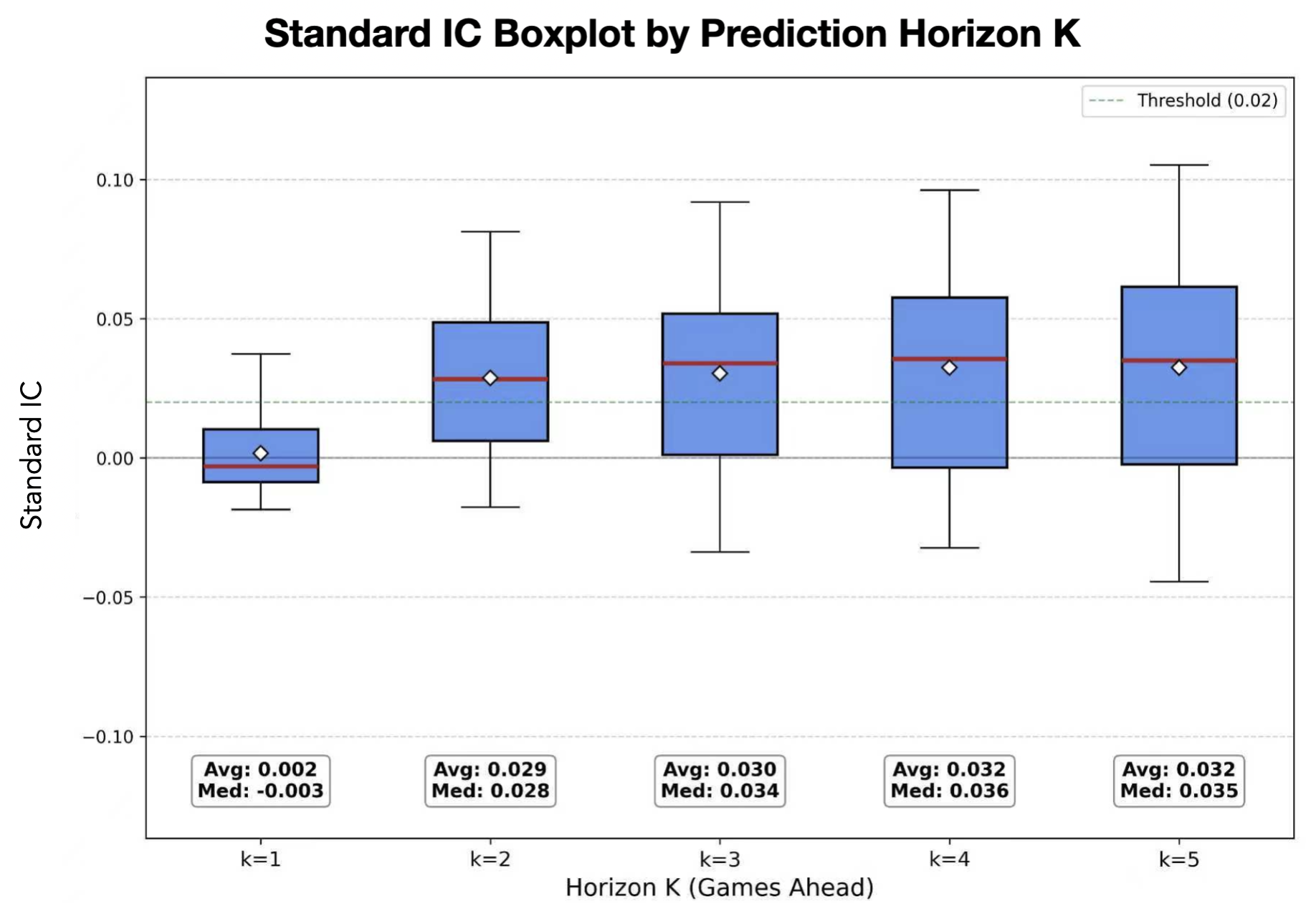}
    \caption{The result of Standard IC}
    \label{fig:ick}
\end{figure}

\section{Results}

\subsection{Performance on Trend Alignment and Signal Correlation}

We first evaluate the effectiveness of the DD-Elo correction mechanism ($\Delta$) in capturing the true skill evolution trend using Area Improvement Percentage (AIP), Directional Accuracy (DA), and Variant IC. The distributions of these metrics across the test segments are presented in Fig.~\ref{fig:sum}.

\subsubsection{Area Improvement Percentage (AIP)}

As shown in Fig.~\ref{fig:sum}, the AIP metric (blue) demonstrates robust performance with a distribution heavily skewed toward 100\%, with a mean value of $74.04\%$ and a median of $88.94\%$. The $50\%$ baseline represents random modification. This indicates that the vast majority of the modification magnitude (in terms of area) is applied in the direction consistent with the player's underlying skill trend. Effectively, DD-Elo maximizes the constructive adjustments while minimizing counterproductive corrections.

\subsubsection{Directional Accuracy (DA)}

The distribution of Directional Accuracy (green) shown in Fig.~\ref{fig:sum} has a mean of $0.534$ and a median of $0.571$, both exceeding the random baseline of $0.5$. This indicates that the directional adjustments produced by DD-Elo are better than random guessing, although the margin is relatively modest due to the substantial noise inherent in individual games. Importantly, the DA results should be interpreted together with the AIP metric. Although some updates are assigned an incorrect direction, the strong AIP performance suggests that these errors are generally small in magnitude and are outweighed by the accumulated effect of correctly directed adjustments. Furthermore, the memory-decay mechanism prevents erroneous directional decisions from exerting a persistent influence, allowing subsequent correct evidence to gradually dominate the overall assessment.

\subsubsection{Variant IC}

Finally, we compute the Variant IC (yellow) to assess the correlation between our correction term and the denoised skill trend in Fig. \ref{fig:sum} . With a mean Variant IC of $0.36$ and median of $0.42$,  the results are an order of magnitude higher than the standard industry significance threshold of $0.02$. This correlation confirms that $\Delta$ is driven by structural changes in a player's performance level.

\begin{figure}
    \centering
    \includegraphics[width=0.93\linewidth]{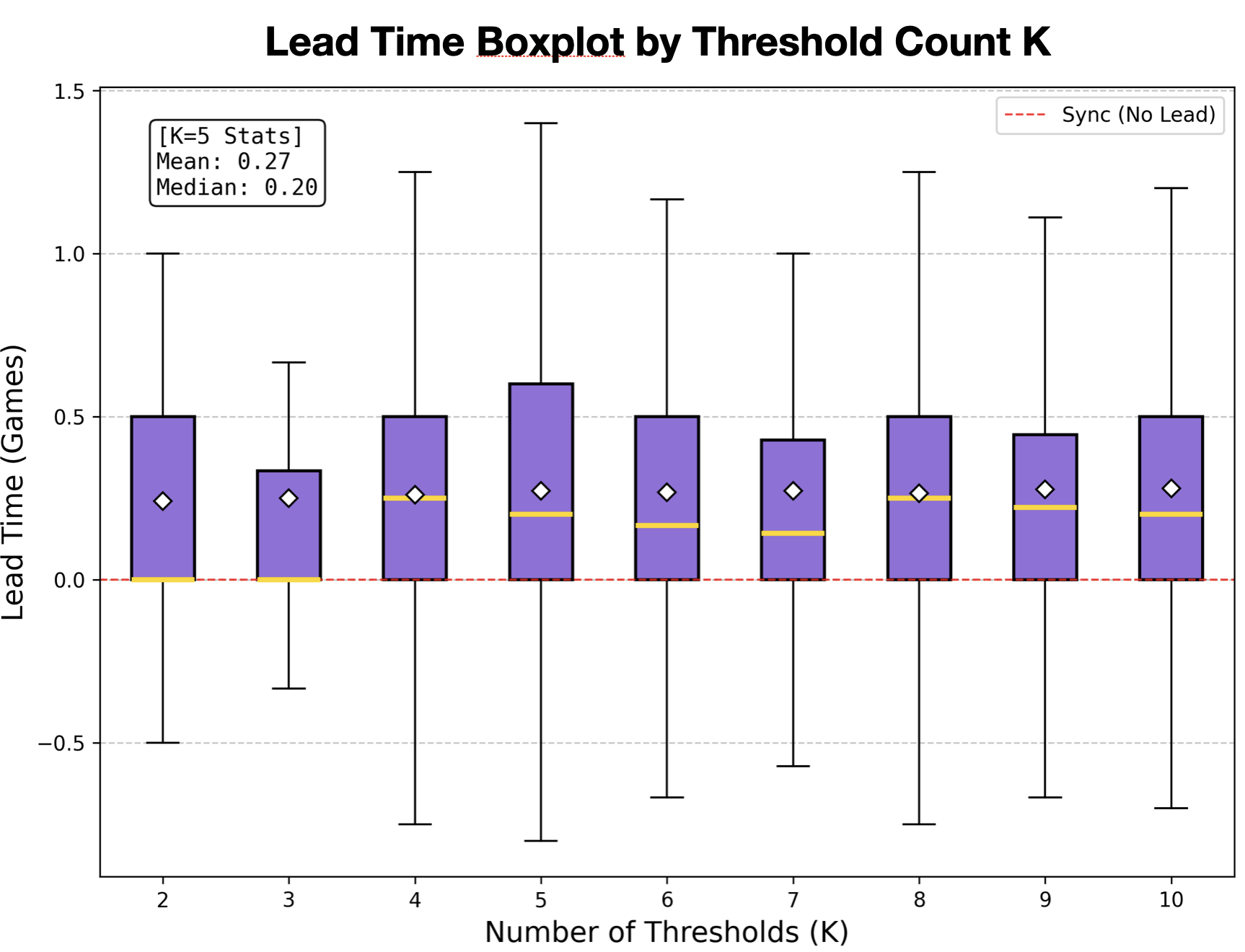}
    \caption{The result of Average Lead Time (ALT)}
    \label{fig:alt}
\end{figure}

\subsection{Predictive Power: Standard IC Analysis}

While Variant IC measures alignment with the latent signal, Standard IC evaluates practical predictive power of $\Delta$ against realized future rating changes. Fig.~\ref{fig:ick} displays the distribution of Standard IC values across different prediction horizons ($K$).

When the horizon extends to $K \ge 2$, the mean Standard IC values rise to approximately $0.030$ and stabilize, with the medians consistently surpassing the significance threshold of $0.02$ (marked by the green dashed line). At $K=1$, the mean Standard IC is near zero ($0.002$), as the outcome of a single game is binary and highly stochastic, dominated by noise that obscures the skill trend. The results imply that the accumulated correction term $\Delta$ is a significant predictor of future rating variations over a multi-game horizon, validating its forward-looking capability.

\subsection{Temporal Advantage: Average Lead Time}

Finally, we quantify the speed of adaptation using the Average Lead Time (ALT), which measures how much earlier DD-Elo reaches new rating milestones than standard Elo. Fig.~\ref{fig:alt} presents the distribution of ALT values calculated across varying resolution thresholds ($K$).

Across all thresholds, the distribution of ALT is positively skewed, with a global mean lead time of $0.28$ games and a median of $0.20$ games. The positive values across the majority of the distribution (indicated by the box positions relative to the zero line) demonstrate that DD-Elo consistently reaches new rating milestones earlier than the standard Elo system. While the lead time of a fraction of a game may appear small, in the context of high-frequency rating updates, this represents a consistent temporal advantage, allowing the system to reflect a player's true strength with reduced lag.

\section{Discussion}

Although focused on chess, the core mechanism of DD-Elo, namely diffusion-based aggregation of intra-game evidence, readily extends to other domains and rating systems. The central requirement is simply the availability of sequential, move-level performance signals. In perfect-information games such as Go or shogi \cite{silver2016mastering,silver2018general}, existing engines provide analogous evaluations. For imperfect-information games such as poker, probabilistic signals derived from counterfactual regret \cite{zinkevich2007regret,lockhart2019computing} can serve as the drift evidence.

Furthermore, our approach is orthogonal to the specific baseline rating framework. Uncertainty-aware models (e.g., Glicko \cite{glickman1995glicko}, TrueSkill \cite{herbrich2006trueskill,minka2018trueskill}) can seamlessly integrate this diffusion-derived signal. Crucially, established in Section \ref{TA}, any rating system adopting an additive, bounded update form $R_{t+1} = R_t + \eta_t,|\eta_t| \le C$ can incorporate our algorithm. This guarantees that inter-game memory and intra-game corrections can be layered onto existing systems while strictly preserving their theoretical stability and bounded-deviation properties.

\section{Conclusion}

In this paper, we propose DD-Elo, a drift–diffusion–based extension of the Elo rating system that addresses the intrinsic response lag of outcome-centric skill assessment by incorporating move-level performance information. Extensive experiments on large-scale real-world chess data demonstrate that DD-Elo adapts to non-stationary skill changes significantly faster than Elo, while preserving bounded deviation, directional accuracy, and long-term alignment with the original rating system. The proposed framework provides a general and explainable approach for enhancing skill rating systems and may be extended to other games, performance modalities, and rating algorithms that require fast yet stable adaptation.

\vspace{0.1cm}
\noindent\textbf{Acknowledgments:}
This work was supported by the Brain Science and Brain-like Intelligence Technology-National Science and Technology Major Project (grant no. 2021ZD0203705) to T.Y.

\bibliographystyle{IEEEtran}
\bibliography{references}

\end{document}